\def\year{2018}\relax
\documentclass[letterpaper]{article} %DO NOT CHANGE THIS
\usepackage{aaai18}  %Required
\usepackage{times}  %Required
\usepackage{helvet}  %Required
\usepackage{courier}  %Required
\usepackage{url}  %Required
\usepackage{graphicx}  %Required
\usepackage{algorithm}
\usepackage{algorithmic}
\usepackage{amsmath,amssymb,amsthm}
\usepackage{epstopdf}
\usepackage{booktabs}
\usepackage{makecell}
\usepackage{bm}
\usepackage{multirow}
\usepackage{subcaption}
\begin{document}
% The file aaai.sty is the style file for AAAI Press
% proceedings, working notes, and technical reports.
%
\title{Deeper Insights into Graph Convolutional Networks \\for Semi-Supervised Learning}
% \author{Qimai Li \\
% The Hong Kong Polytechnic Univercity\\
% csqmli@comp.polyu.edu.hk\\
% \And  Zhichao Han \\
% ETH Zurich\\
% zhhan@student.ethz.ch\\
% \And Xiao-Ming Wu \\
% The Hong Kong Polytechnic Univercity\\
% xiao-ming.wu@polyu.edu.hk\\}

\author{ Qimai Li\textsuperscript{1}, Zhichao Han\textsuperscript{12}, Xiao-Ming Wu\textsuperscript{1}\thanks{Corresponding author.}\\
    \textsuperscript{1}{The Hong Kong Polytechnic University}\\
    \textsuperscript{2}{ETH Zurich}\\
    csqmli@comp.polyu.edu.hk,
    zhhan@student.ethz.ch,
    xiao-ming.wu@polyu.edu.hk
}

\maketitle
\begin{abstract}
%Semi-supervised learning is an important paradigm in machine learning and data mining. Numerous researches have demonstrated that utilizing unlabeled data properly with labeled data can significantly improve learning accuracy. But it remains a challenging problem how to effectively use the information of unlabeled data, such as feature attributes and connectivity patterns, especially when labeled data is very scarce. In this paper, we address this issue by proposing a new approach for semi-supervised classification. Inspired by co-training, we first use random walks to explore the connectivity patterns of labeled and unlabeled data, and then apply graph convolutional networks to classify the unlabeled data based on their feature attributes. By successfully combing the merits of graph-based methods and neural network based methods, our approach achieves significant improvement over state-of-the-art approaches especially whenx labeled data is very limited, which is verified by extensive experiments on real benchmark datasets.

Many interesting problems in machine learning are being revisited with new deep learning tools. For graph-based semi-supervised learning, a recent important development is graph convolutional networks (GCNs), which nicely integrate local vertex features and graph topology in the convolutional layers. Although the GCN model compares favorably with other state-of-the-art methods, its mechanisms are not clear and it still requires considerable amount of labeled data for validation and model selection.

\indent In this paper, we develop deeper insights into the GCN model and address its fundamental limits. First, we show that the graph convolution of the GCN model is actually a special form of Laplacian smoothing, which is the key reason why GCNs work, but it also brings potential concerns of over-smoothing with many convolutional layers. Second, to overcome the limits of the GCN model with shallow architectures, we propose both co-training and self-training approaches to train GCNs. Our approaches significantly improve GCNs in learning with very few labels, and exempt them from requiring additional labels for validation. Extensive experiments on benchmarks have verified our theory and proposals.

\end{abstract}
\section{Introduction}\label{sec:introduction}

\noindent The breakthroughs in deep learning have led to a paradigm shift in artificial intelligence and machine learning. On the one hand, numerous old problems have been revisited with deep neural networks and
huge progress has been made in many tasks previously seemed out of reach, such as machine translation and computer vision. On the other hand, new techniques such as geometric deep learning \cite{bronstein2017geometric} are being developed to generalize deep neural models to new or non-traditional domains.

It is well known that training a deep neural model typically requires a large amount of labeled data, which cannot be satisfied in many scenarios due to the high cost of labeling training data. To reduce the amount of data needed for training, a recent surge of research interest has focused on few-shot learning \cite{lake2015human,rezende2016one} -- to learn a classification model with very few examples from each class. Closely related to few-shot learning is semi-supervised learning, where a large amount of \emph{unlabeled} data can be utilized to train with typically a small amount of labeled data.

%\noindent The breakthrough in deep learning has lead to a paradigm shift in machine learning and artificial intelligence. Nevertheless, training a deep neural network typically requires a large amount of labeled data, which cannot be satisfied in many scenarios. A recent surge of research interest has focused on few-shot learning -- to learn a classification model with very few examples from each class. Closely related to few-shot learning is semi-supervised learning, where unlabeled data is expected to be utilized to aid classification with limited amount of labeled data.

%Leveraging unlabeled data in training can improve learning accuracy significantly if used properly, which helps reduce a great amount of time and human efforts in annotation. This makes semi-supervised learning an important research area that has attracted a lot of research efforts in the past two decades. Nevertheless, developing effective methods to maximize the utilization of the information of unlabel data, such as the graph or manifold structure and the feature attributes, still remains challenging especially when labeled data is scarce. As unlabeled data can only help when right assumptions are made about the data \cite{zhu2009introduction}.

Many researches have shown that leveraging unlabeled data in training can improve learning accuracy significantly if used properly \cite{zhu2009introduction}. The key issue is to maximize the effective utilization of structural and feature information of unlabeled data. Due to the powerful feature extraction capability and recent success of deep neural networks, there have been some successful attempts to revisit semi-supervised learning with neural-network-based models, including ladder network \cite{rasmus2015semi}, semi-supervised embedding \cite{weston2012deep}, planetoid \cite{yang2016revisiting}, and graph convolutional networks \cite{kipf2016semi}.

The recently developed graph convolutional neural networks (GCNNs) \cite{defferrard2016convolutional} is a successful attempt of generalizing the powerful convolutional neural networks (CNNs) in dealing with Euclidean data to modeling graph-structured data. In their pilot work \cite{kipf2016semi}, Kipf and Welling proposed a simplified type of GCNNs, called graph convolutional networks (GCNs), and applied it to semi-supervised classification. The GCN model naturally integrates the connectivity patterns and feature attributes of graph-structured data, and outperforms many state-of-the-art methods significantly on some benchmarks. Nevertheless, it suffers from similar problems faced by other neural-network-based models. The working mechanisms of the GCN model for semi-supervised learning are not clear, and the training of GCNs still requires considerable amount of labeled data for parameter tuning and model selection, which defeats the purpose for semi-supervised learning.

In this paper, we demystify the GCN model for semi-supervised learning. In particular, we show that the graph convolution of the GCN model is simply a special form of Laplacian smoothing, which mixes the features of a vertex and its nearby neighbors. The smoothing operation makes the features of vertices in the same cluster similar, thus greatly easing the classification task, which is the key reason why GCNs work so well. However, it also brings potential concerns of over-smoothing. If a GCN is deep with many convolutional layers, the output features may be over-smoothed and vertices from different clusters may become indistinguishable. The mixing happens quickly on small datasets with only a few convolutional layers, as illustrated by Fig.~\ref{fig:karate}. Also, adding more layers to a GCN will make it much more difficult to train.

However, a shallow GCN model such as the two-layer GCN used in \cite{kipf2016semi} has its own limits. Besides that it requires many additional labels for validation, it also suffers from the localized nature of the convolutional filter. When only few labels are given, a shallow GCN cannot effectively propagate the labels to the entire data graph. As illustrated in Fig.~\ref{fig:gcn_validation}, the performance of GCNs drops quickly as the training size shrinks, even for the one with 500 additional labels for validation.

\begin{figure}[t!]
\centering
\includegraphics [width=8cm]{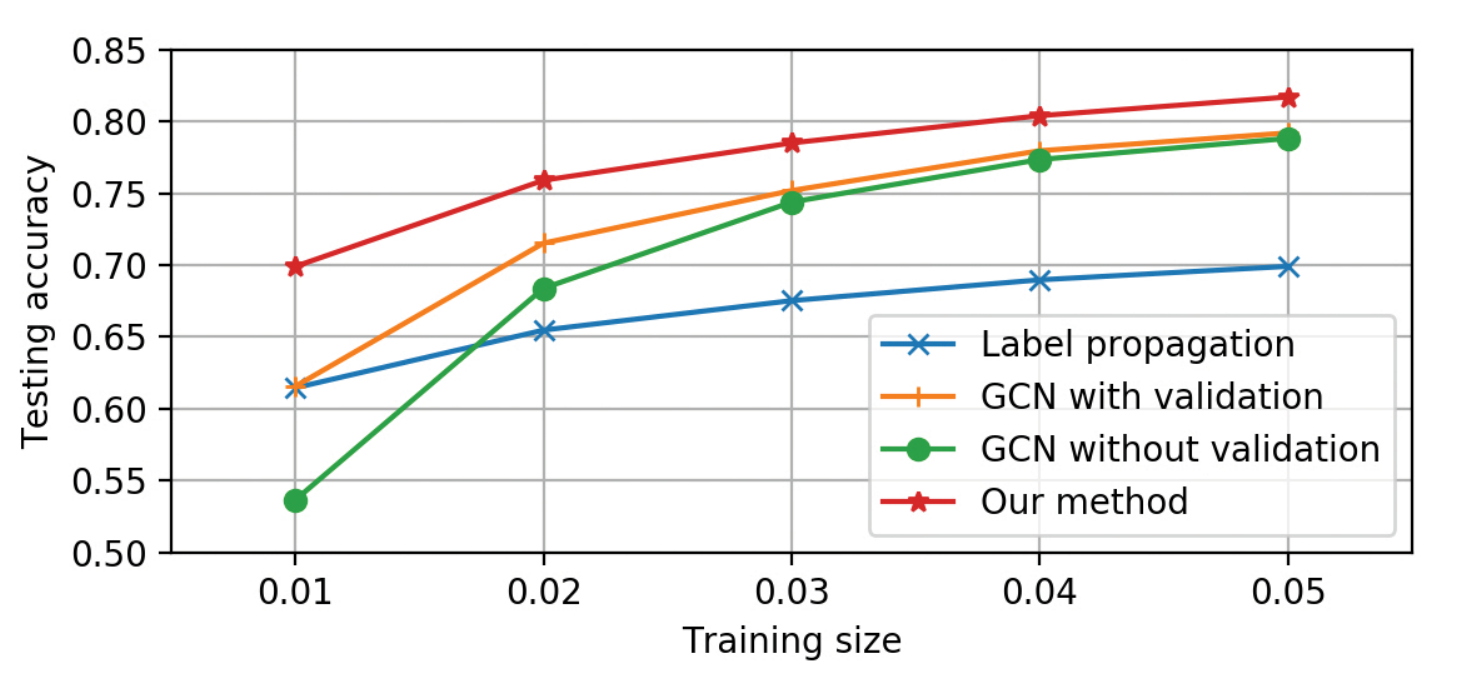}
\caption{Performance comparison of GCNs, label propagation, and our method for semi-supervised classification on the Cora citation network.}\label{fig:gcn_validation}
\end{figure}

To overcome the limits and realize the full potentials of the GCN model, we propose a co-training approach and a self-training approach to train GCNs. By co-training a GCN with a random walk model, the latter could complement the former in exploring global graph topology. By self-training a GCN, we can exploit its feature extraction capability to overcome its localized nature. Combining both the co-training and self-training approaches can substantially improve the GCN model for semi-supervised learning with very few labels, and exempt it from requiring additional labeled data for validation. As illustrated in Fig.~\ref{fig:gcn_validation}, our method outperforms GCNs by a large margin.

%To realize the full potential of GCNs, we propose to co-train a GCN with a random walk model and self-train a GCN to overcome its limits. By combining the advantage of random walk models for exploring the global graph topology and the advantage of GCNs for feature extraction, our a pproach substantially improve GCNs for semi-supervised learning with few labels, and without requiring any additional labeled data for validation.

%To realize the full potential of GCNs, we propose to co-train a GCN with a random walk model and self-train a GCN to overcome its limits. By combining the advantage of random walk models for exploring the global graph topology and the advantage of GCNs for feature extraction, our approach substantially improve GCNs for semi-supervised learning with few labels, and without requiring any additional labeled data for validation.

In a nutshell, the key innovations of this paper are: 1) providing new insights and analysis of the GCN model for semi-supervised learning; 2) proposing solutions to improve the GCN model for semi-supervised learning. The rest of the paper is organized as follows. Section~\ref{sec:GCN} introduces the preliminaries and related works. In Section~\ref{sec:Analysis}, we analyze the mechanisms and fundamental limits of the GCN model for semi-supervised learning. In Section~\ref{sec:Solution}, we propose our methods to improve the GCN model. In Section~\ref{sec:Experiment}, we conduct experiments to verify our analysis and proposals. Finally, Section~\ref{sec:conclusion} concludes the paper.

\section{Preliminaries and Related Works}\label{sec:GCN}
% preliminaries section

First, let us define some notations used throughout this paper. A graph is represented by $\mathcal{G}=(\mathcal{V},\mathcal{E})$, where $\mathcal{V}$ is the vertex set with $|\mathcal{V}| = n$ and $\mathcal{E}$ is the edge set. In this paper, we consider undirected graphs. Denote by $A=[a_{ij}]\in{\mathbb R}^{n\times n}$ the adjacency matrix which is nonnegative. Denote by $D=\mathrm{diag}(d_1,d_2,\ldots,d_n)$ the degree matrix of $A$ where $d_i=\sum_{j}a_{ij}$ is the degree of vertex $i$. The graph Laplacian \cite{Chung97} is defined as $L:=D-A$, and the two versions of normalized graph Laplacians are defined as $L_{\mathrm{sym}}:=D^{-\frac{1}{2}}LD^{-\frac{1}{2}}$ and $L_{\mathrm{rw}}: = D^{-1}L$ respectively.

\subsection{Graph-Based Semi-Supervised Learning}

The problem we consider in this paper is semi-supervised classification on graphs. Given a graph $\mathcal{G}=(\mathcal{V},\mathcal{E}, X)$, where $X=[\mathbf{x}_1,\mathbf{x}_2,\cdots,\mathbf{x}_n]^\top \in R^{n\times c}$ is the feature matrix, and $\mathbf{x}_i \in R^{c}$ is the $c$-dimensional feature vector of vertex $i$. Suppose that the labels of a set of vertices $\mathcal{V}_l$ are given, the goal is to predict the labels of the remaining vertices $\mathcal{V}_u$.

Graph-based semi-supervised learning has been a popular research area in the past two decades. By exploiting the graph or manifold structure of data, it is possible to learn with very few labels. Many graph-based semi-supervised learning methods make the cluster assumption \cite{chapelle2005semi}, which assumes that nearby vertices on a graph tend to share the same label. Researches along this line include min-cuts \cite{blum2001learning} and randomized min-cuts \cite{blum2004semi}, spectral graph transducer \cite{joachims2003transductive}, label propagation \cite{Zhu03} and its variants \cite{Zhou03,bengio2006label}, modified adsorption \cite{talukdar2009new}, and iterative classification algorithm \cite{sen2008collective}.

%In this paper,
%we mainly focus on the setting that a graph is explicitly
%given and represents additional information not present in
%the feature vectors (e.g., the graph edges correspond to hyperlinks
%between documents, rather than distances between
%the bag-of-words representation of a document).

But the graph only represents the structural information of data. In many applications, data instances come with feature vectors containing information not present in the graph. For example, in a citation network, the citation links between documents describe their citation relations, while the documents are represented as bag-of-words vectors which describe their contents. Many semi-supervised learning methods seek to jointly model the graph structure and feature attributes of data. A common idea is to regularize a supervised learner with some regularizer. For example, manifold regularization (LapSVM) \cite{belkin2006manifold} regularizes a support vector machine with a Laplacian regularizer. Deep semi-supervised embedding \cite{weston2012deep} regularizes a deep neural network with an embedding-based regularizer. Planetoid \cite{yang2016revisiting} also regularizes a neural network by jointly predicting the class label and the context of an instance.

\subsection{Graph Convolutional Networks}

Graph convolutional neural networks (GCNNs) generalize traditional convolutional neural networks to the graph domain. There are mainly two types of GCNNs \cite{bronstein2017geometric}: spatial GCNNs and spectral GCNNs. Spatial GCNNs view the convolution as ``patch operator'' which constructs a new feature vector for each vertex using its neighborhood information. Spectral GCNNs define the convolution by decomposing a graph signal $\mathbf{s} \in R^n$ (a scalar for each vertex) on the spectral domain and then applying a spectral filter $g_{\theta}$ (a function of eigenvalues of $L_{\mathrm{sym}}$) on the spectral components \cite{bruna2013spectral,sandryhaila2013discrete,shuman2013emerging}. However this model requires explicitly computing the Laplacian eigenvectors, which is impractical for real large graphs. A way to circumvent this problem is by approximating the spectral filter $g_{\theta}$ with Chebyshev polynomials up to $K^{th}$ order \cite{hammond2011wavelets}. In \cite{defferrard2016convolutional}, Defferrard et al. applied this to build a $K$-localized ChebNet, where the convolution is defined as:
\begin{equation}\label{eq:cheby}
g_{\theta} \star \mathbf{s} \approx  \sum_{k=0}^K \theta_k' T_k(L_{\mathrm{sym}}) \mathbf{s},
% g_{\theta} \star s = U g_{\theta}(L_{\mathrm{sym}}) U^T s \approx  \sum_{k=0}^K \theta_k' T_k(L_{\mathrm{sym}}) s
\end{equation}
where $\mathbf{s} \in R^n$ is the signal on the graph, $g_{\theta}$ is the spectral filter, $\star$ denotes the convolution operator, $T_k$ is the Chebyshev polynomials, and $\theta'\in R^K$ is a vector of Chebyshev coefficients. By the approximation, the ChebNet is actually spectrum-free.

% Based on different designs of convolution operation, GCN can be roughly divided into two types: spectral GCN and spatial GCN\cite{bronstein2017geometric}. Spatial GCN models view the convolution operation as ``patch operator'' which constructs new feature for each vertex using its neighborhood information. Different spatial GCN models could be referred in \cite{bronstein2017geometric}. Spectral GCN models define convolution by firstly decomposing the given graph signal on spectral domain and then applying a spectral filter to these spectral components. The first work appeared in \cite{bruna2013spectral}. However this model requires explicitly computing the Laplacian eigenvectors which is impractical for real graphs. Defferrard et al. \cite{defferrard2016convolutional} circumvented this problem by approximating the spectral filter $g_{\theta}$ using Chebyshev polynomials $T_k(x)$ up to the $K^{th}$ order, such that

In \cite{kipf2016semi}, Kipf and Welling simplified this model by limiting $K=1$ and approximating the largest eigenvalue $\lambda_{max}$ of $L_{\mathrm{sym}}$ by 2. In this way, the convolution becomes
\begin{equation}
\label{eq:kipfnew}
g_{\theta} \star \mathbf{s} = \theta \left(I + D^{-\frac{1}{2}} A D^{-\frac{1}{2}} \right) \mathbf{s},
\end{equation}
where $\theta$ is the only Chebyshev coefficient left. They further applied a normalization trick to the convolution matrix:
\begin{equation}
\label{eq:normalization}
I + D^{-\frac{1}{2}} A D^{-\frac{1}{2}} \rightarrow  \tilde{D}^{-\frac{1}{2}}\tilde{A}\tilde{D}^{-\frac{1}{2}},
\end{equation}
where $\tilde{A}=A+I$ and $\tilde{D}=\sum_j \tilde{A}_{ij} $.

Generalizing the above definition of convolution to a graph signal with $c$ input channels, i.e., $X \in R^{n \times c}$ (each vertex is associated with a $c$-dimensional feature vector), and using $f$ spectral filters, the propagation rule of this simplified model is:
\begin{equation}
\label{eq:convolution}
H^{(l+1)}= \sigma \left(
    \tilde{D}^{-\frac{1}{2}}\tilde{A}\tilde{D}^{-\frac{1}{2}} H^{(l)}\Theta^{(l)}
    \right),
\end{equation}
where $H^{(l)}$ is the matrix of activations in the $l$-th layer, and $H^{(0)} = X$, $\Theta^{(l)}\in R^{c \times f}$ is the trainable weight matrix in layer $l$, $\sigma$ is the activation function, e.g., $ReLU(\cdot) = max(0, \cdot)$.

This simplified model is called graph convolutional networks (GCNs), which is the focus of this paper.

% And Kipf and Welling applied a \emph{normalization trick} to Eq (\ref{eq:kipfnew}): $I + D^{-\frac{1}{2}} A D^{-\frac{1}{2}} \rightarrow \hat{A} = \tilde{D}^{-\frac{1}{2}}\tilde{A}\tilde{D}^{-\frac{1}{2}}$ where $\tilde{A}=A+I$ and $\tilde{D}$ is the degree matrix of $\tilde{A}$. With this trick, the propagation rule used in GCN is following

\subsection{Semi-Supervised Classification with GCNs}

% As defined before, semi-supervised classification refers to the task of classifying vertices in a graph, where only a small number of label is already known. The key point of doing classification is how to jointly utilize the connectivity patterns and feature attributes to classify unknown vertices.

% GCNs is a natural choice for graph-based semi-supervised classification task as it integrates feature attributes $X$ and the underlying graph structure encoded by $\hat{A} \triangleq \tilde{D}^{-\frac{1}{2}}\tilde{A}\tilde{D}^{-\frac{1}{2}}$ into the new feature representation (see equation \ref{eq:convolution}). In GCNs, every vertex propagates its feature to its adjacent vertices via the convolution operation. With the multi-layer structure, the features are propagated across the graph. We will see in the next section that this learning process is actually the Laplacian smoothing.

In \cite{kipf2016semi}, the GCN model was applied for semi-supervised classification in a neat way. The model used is a two-layer GCN which applies a \emph{softmax} classifier on the output features:
\begin{equation}
\label{eq:kipfGCN}
Z = \text{softmax}\left (
    \hat{A}\; ReLU\left(
    \hat{A}X\Theta^{(0)}
    \right)\Theta^{(1)}
\right),
\end{equation}
where $\hat{A}=\tilde{D}^{-\frac{1}{2}}\tilde{A}\tilde{D}^{-\frac{1}{2}}$, $\text{softmax}(x_i) = \frac{1}{\mathcal{Z}} exp(x_i)$ with $\mathcal{Z} = \sum_i exp(x_i)$. The loss function is defined as the cross-entropy error over all labeled examples:
\begin{equation}
\label{eq:loss}
\mathcal{L} :=  -\sum_{i\in \mathcal{V}_l} \sum_{f=1}^F Y_{if} \ln Z_{if},
\end{equation}
where $\mathcal{V}_l$ is the set of indices of labeled vertices and $F$ is the dimension of the output features, which is equal to the number of classes. $Y\in R^{|\mathcal{V}_l|\times F}$ is a label indicator matrix. The weight parameters $\Theta^{(0)}$ and $\Theta^{(1)}$ can be trained via gradient descent.

The GCN model naturally combines graph structures and vertex features in the convolution, where the features of unlabeled vertices are mixed with those of nearby labeled vertices, and propagated over the graph through multiple layers. It was reported in \cite{kipf2016semi} that GCNs outperformed many state-of-the-art methods significantly on some benchmarks such as citation networks.

%GCNs is a natural choice for graph-based semi-supervised classification task as it integrates feature attributes $X$ and the underlying graph structure encoded by $\hat{A} \triangleq \tilde{D}^{-\frac{1}{2}}\tilde{A}\tilde{D}^{-\frac{1}{2}}$ into the new feature representation (see equation (\ref{eq:convolution}) ). In GCNs, every vertex propagates its feature to its adjacent vertices via the convolution operation. And the features of vertices in same cluster will be closer. We will see in the next section that this learning process is actually the Laplacian smoothing.

\section{Analysis}\label{sec:Analysis}

Despite its promising performance, the mechanisms of the GCN model for semi-supervised learning have not been made clear. In this section, we take a closer look at the GCN model, analyze why it works, and point out its limitations.

%In this section, we will show the reason why GCNs can achieve such good performance. Specifically, we will demonstrate that the number of convolution layers cannot be too many. Finally, we point out the limitations with such architecture.

% This demonstrates the effectiveness of GCN in combining features and graph connectivity patterns for prediction.

\begin{table}
\centering
\caption{GCNs vs. Fully-connected networks}\label{tb:one_layer_gcn}
\begin{tabular}{cccc}
  \thead{ One-layer \\ FCN} & \thead{ Two-layer \\ FCN} & \thead{ One-layer \\ GCN} & \thead{ Two-layer \\ GCN}\\
  \midrule
   0.530860 &  0.559260 &  \textbf{0.707940} &  \textbf{0.798361}\\
\end{tabular}
\end{table}

\subsection{Why GCNs Work}

To understand the reasons why GCNs work so well, we compare them with the simplest fully-connected  networks (FCNs), where the layer-wise propagation rule is
\begin{equation}
H^{(l+1)} = \sigma \left( H^{(l)}\Theta^{(l)} \right ).
\end{equation}
Clearly the only difference between a GCN and a FCN is the graph convolution matrix $\hat{A}=\tilde{D}^{-\frac{1}{2}}\tilde{A}\tilde{D}^{-\frac{1}{2}}$ (Eq.~(\ref{eq:kipfGCN})) applied on the left of the feature matrix $X$. To see the impact of the graph convolution, we tested the performances of GCNs and FCNs for semi-supervised classification on the Cora citation network with 20 labels in each class. The results can be seen in Table \ref{tb:one_layer_gcn}. Surprisingly, even a one-layer GCN outperformed a one-layer FCN by a very large margin.

%We test the performances of the GCNs and the fully-connected neural network on the CORA dataset. The results are reported in Table \ref{tb:one_layer_gcn}. It is obvious that GCNs outperform the simple fully-connected neural network by a significant margin. In order to understand the reason which makes GCNs so powerful, let us compare the layer-wise propagation rule of GCNs (eq. (\ref{eq:convolution})) with the propagation rule of the fully-connected networks,
%\begin{equation}
%H^{(l+1)} = \sigma \left( H^{(l)}\Theta^{(l)} \right )
%\end{equation}
\begin{figure*}[ht!]
\centering
\begin{subfigure}{.2\linewidth}
  \centering
  \includegraphics[width=\textwidth]{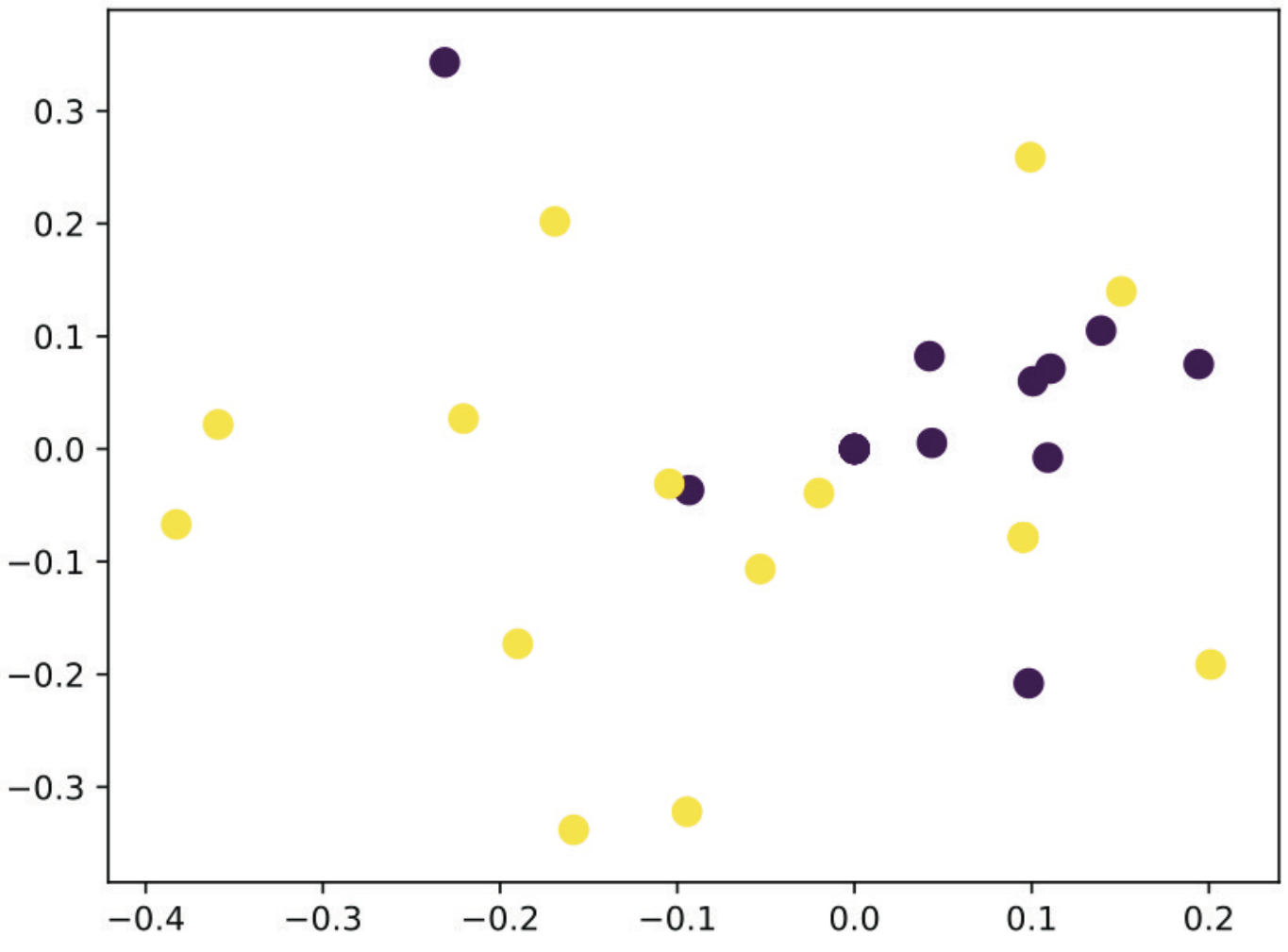}
  \caption{1-layer}
  \label{fig:sub1}
\end{subfigure}%
% \hfill
\begin{subfigure}{.2\linewidth}
  \centering
  \includegraphics[width=\textwidth]{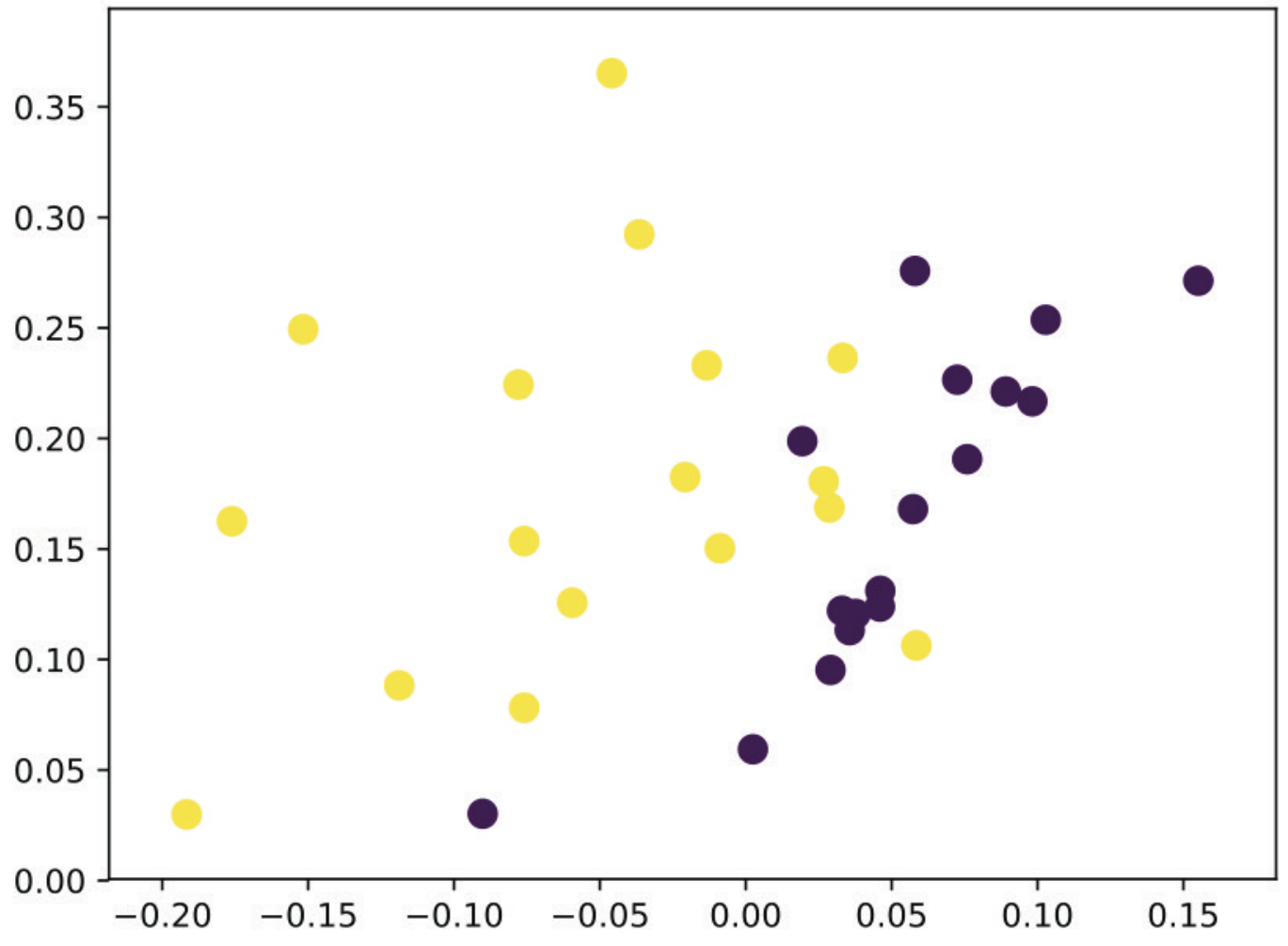}
  \caption{2-layer}
  \label{fig:sub2}
\end{subfigure}%
% \hfill
\begin{subfigure}{.2\linewidth}
  \centering
  \includegraphics[width=\textwidth]{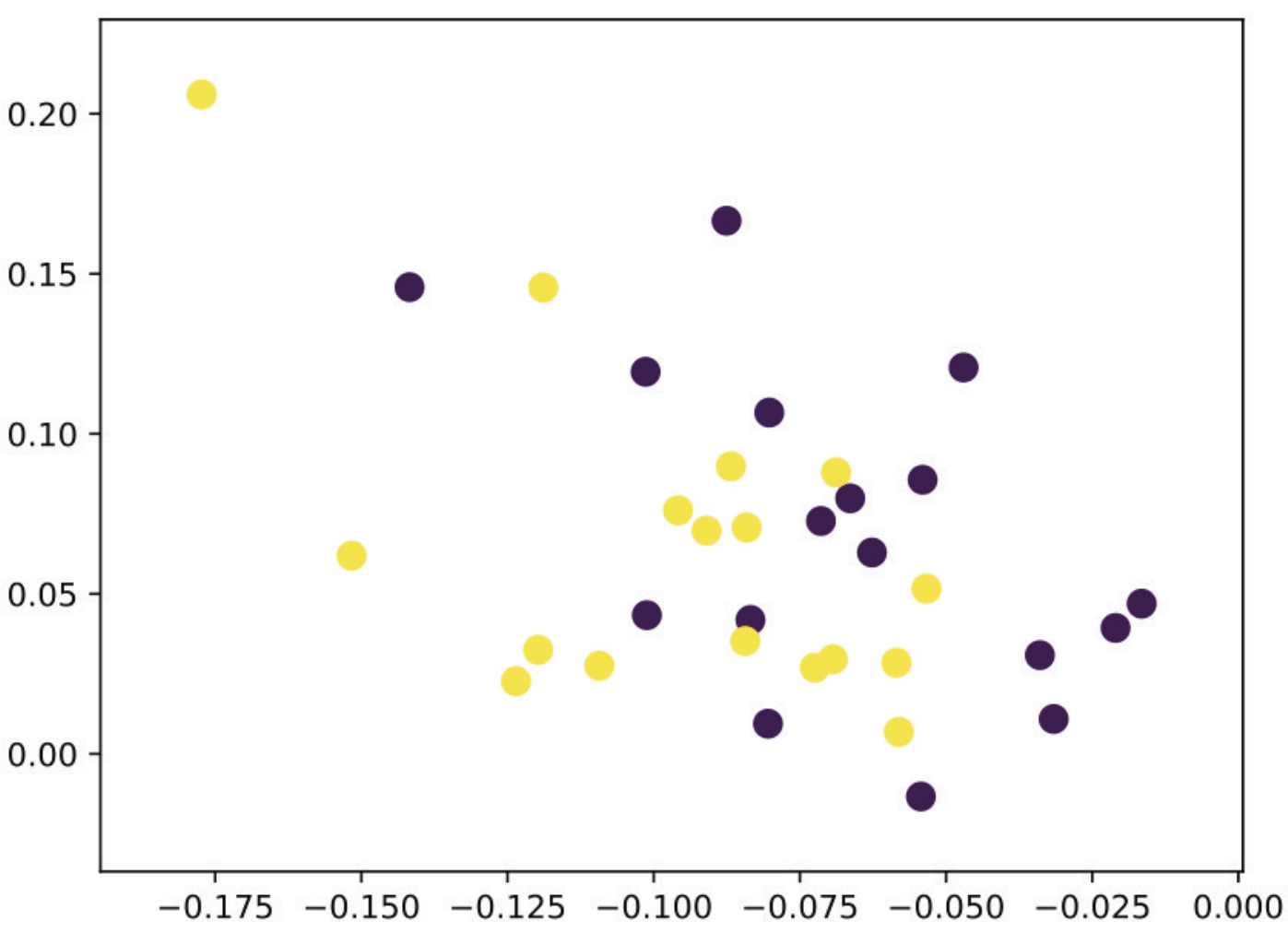}
  \caption{3-layer}
  \label{fig:sub3}
\end{subfigure}%
% \hfill
\begin{subfigure}{.2\linewidth}
  \centering
  \includegraphics[width=\textwidth]{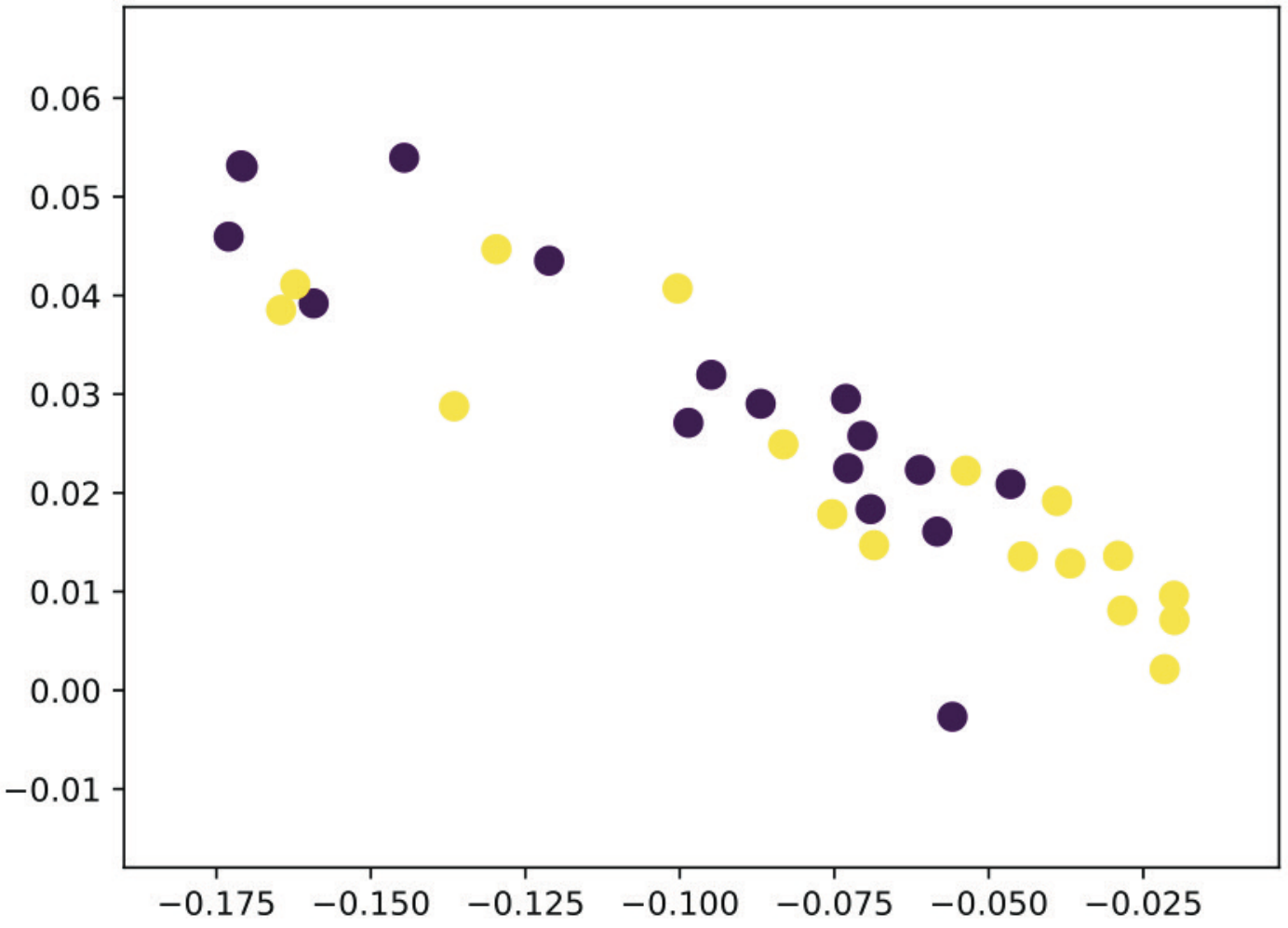}
  \caption{4-layer}
  \label{fig:sub4}
\end{subfigure}%
% \hfill
\begin{subfigure}{.2\linewidth}
  \centering
  \includegraphics[width=\textwidth]{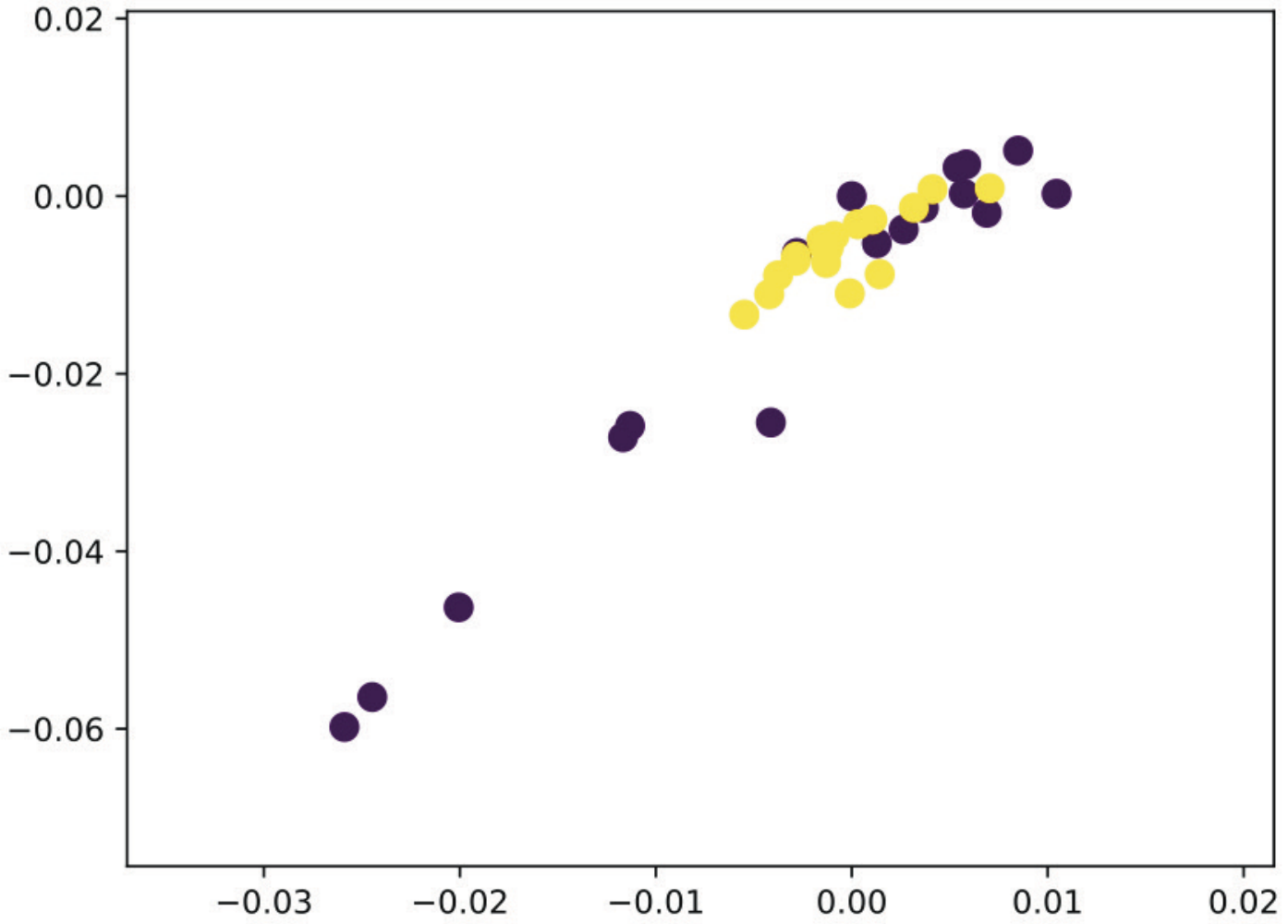}
  \caption{5-layer}
  \label{fig:sub5}
\end{subfigure}%
\caption{Vertex embeddings of Zachary's karate club network with GCNs with 1,2,3,4,5 layers.}
\label{fig:karate}
\end{figure*}

\textbf{Laplacian Smoothing.} Let us first consider a one-layer GCN. It actually contains two steps. 1) Generating a new feature matrix $Y$ from $X$ by applying the graph convolution:
\begin{equation}
    Y=\tilde{D}^{-1/2}\tilde{A}\tilde{D}^{-1/2}X \label{eq:conv}.
\end{equation}
2) Feeding the new feature matrix $Y$ to a fully connected layer. Clearly the graph convolution is the key to the huge performance gain.

Let us examine the graph convolution carefully. Suppose that we add a self-loop to each vertex in the graph, then the adjacency matrix of the new graph is $\tilde{A}=A+I$. The Laplacian smoothing \cite{taubin1995signal} on each channel of the input features is defined as:
\begin{equation}
    \mathbf{\hat{y}}_i= (1-\gamma)\mathbf{x}_i + \gamma\sum_{j}\frac{\tilde{a}_{ij}}{d_i}\mathbf{x}_j \qquad(\mathrm{for} ~~1\leq i \leq n),
\end{equation}
where $0<\gamma\le1$ is a parameter which controls the weighting between the features of the current vertex and the features of its neighbors. We can write the Laplacian smoothing in matrix form:
\begin{equation}
    \label{eq:Laplacian smoothing}
    \hat{Y}=X-\gamma \tilde{D}^{-1}\tilde{L} X = (I-\gamma\tilde{D}^{-1}\tilde{L})X,
\end{equation}
where $\tilde{L}=\tilde{D}-\tilde{A}$. By letting $\gamma=1$, i.e., only using the neighbors' features, we have $\hat{Y}=\tilde{D}^{-1}\tilde{A}X$, which is the standard form of Laplacian smoothing.

Now if we replace the normalized Laplacian $\tilde{D}^{-1}\tilde{L}$ with the symmetrically normalized Laplacian $\tilde{D}^{-\frac{1}{2}}\tilde{L}\tilde{D}^{-\frac{1}{2}}$ and let $\gamma=1$, we have $\hat{Y}=\tilde{D}^{-1/2}\tilde{A}\tilde{D}^{-1/2}X$, which is exactly the graph convolution in Eq.~(\ref{eq:conv}). We thus call the graph convolution a special form of Laplacian smoothing -- symmetric Laplacian smoothing. Note that here the smoothing still includes the current vertex's features, as each vertex has a self-loop and is its own neighbor.

%Eq.~(\ref{eq:Laplacian smoothing}) becomes
%\begin{equation}
%\label{eq:Lap-2}
%\hat{Y}=\tilde{D}^{-1}\tilde{A}X.
%\end{equation}

%
%We can see that the graph convolution in Eq.~(\ref{eq:conv}) is very similar to the Laplacian smoothing in Eq.~(\ref{eq:Lap-2}), with difference in the way the adjacency matrix $\tilde{A}$ is normalized. We thus call the graph convolution a special form of Laplacian smoothing. Note that here the smoothing still includes the current vertex's features, as each vertex has a self-loop and is its own neighbors.

%Note that the feature of the current vertex itself still plays a role in its smoothed result when $\alpha=1$, because extra self-loops are added to the graph in preprocessing and each vertex is adjacant to itself. So the convolution is indeed the same as Laplacian smoothing.

The Laplacian smoothing computes the new features of a vertex as the weighted average of itself and its neighbors'. Since vertices in the same cluster tend to be densely connected, the smoothing makes their features similar, which makes the subsequent classification task much easier. As we can see from Table \ref{tb:one_layer_gcn}, applying the smoothing only once has already led to a huge performance gain.

%calculates new features as the weighted average of the orginal features and neighbor's features. So the features of adjacent vertices will be closer after smoothing. Note that there is a basic assumption for vertex classification task is that vertices of the same class tend to be closely connected. So the smoothing makes data points of the same cluster closer. As a consequence the classification becomes easier. This is the reason why GCNs outperform fully connected networks so much.

\textbf{Multi-layer Structure.} We can also see from Table \ref{tb:one_layer_gcn} that while the 2-layer FCN only slightly improves over the 1-layer FCN, the 2-layer GCN significantly improves over the 1-layer GCN by a large margin. This is because applying smoothing again on the activations of the first layer makes the output features of vertices in the same cluster more similar and further eases the classification task.

% Laplacian smoothing calculates the new feature as the weighted average of a vertex and its neighbors. After smoothing, the features of adjacent vertices will be closer. $\alpha$ controls the weight of the feature of this vertex itself and features of its neighbors. If we choose $\alpha=1$, i.e. only using neighbors' features, eq. (\ref{eq:Laplacian smoothing}) will become $Y=\tilde{D}^{-1}\tilde{A}X$. This is almost the same as eq. (\ref{eq:conv}), except that adjacency matrix in eq. (\ref{eq:conv}) is normalized in a different way. Remarkably, the feature of a vertex itself still plays a role in its smoothed result when $\alpha=1$, because extra self-loops are added to the graph in preprocessing and each vertex is adjacant to itself. The reason that GCN outperforms fully connected networks so much is due to the smoothing. The basic assumption for vertex classification task is that vertices of the same class tend to be closely connected. So after Laplacian smoothing by eq. (\ref{eq:Laplacian smoothing}) or convolution by eq. (\ref{eq:conv}), data points of the same class will become closer which makes the classification easier.

\subsection{When GCNs Fail}
%%%%%%%%%%%% karate dataset %%%%%%%%%%%%%

% Plot a figure that compares GCNs with and without validation on the Cora citation networks, with 1\%, 2\%, 3\%, 4\%, 5\% labeled data. Also draw a curve of the accuracy of label propagation with the same training data.

% Using this figure, explain two drawbacks of the GCNs used in Kipf's paper: 1) when labeled data is not enough,GCN only explores the local graph structure, as the first-order convolution is essentially a linear combination of feature vectors of a vertex's adjacent neighbors. As a result, it will not work well. 2) Another problem with GCN is that it requires an additional validation set for hyper-parameter tuning and early stopping (using the validation accuracy for model selection). In \cite{kipf2016semi}, the authors used an additional 500 labeled data for validation, which is way more than the labeled data used for training. This is certainly undesirable as labeled data are often scarce and expensive to acquire. Without validation, GCNs cannot work well.

% Based on our above analyses, these limitations will not be overcomed by adding more convolutional layers as applying convolutional operations repeatedly will mix the features of vertices in the same component and makes these features undistinguishable. Besides adding more layers will make the GCN difficult to train due to the gradient vanishing problem.

We have shown that the graph convolution is essentially a type of Laplacian smoothing. A natural question is how many convolutional layers should be included in a GCN? Certainly not the more the better. On the one hand, a GCN with many layers is difficult to train. On the other hand, repeatedly applying Laplacian smoothing may mix the features of vertices from different clusters and make them indistinguishable. In the following, we illustrate this point with a popular dataset.

%We have seen that the convolution is indeed doing Laplacian smoothing. A natural question is how many convolutional layers should be included into GCNs? Certainly, we cannot use too many layers since repeating smoothing too namy times will deteriorate the quality of new feature. This is because smoothing $K$ times propagates the feature of every vertex to its $K$-hop neighbors and every vertex will also be influenced by these $K$-hop neighbors. As there also exist a few edges connecting vertices of different classes, more vertices of different classes will be included into the same neighborhood if $K$ is large. After smoothing, the features of these vertices will become closer and indistinguishable.

% So the shallow GCNs is favored. In fact, \cite{kipf2016semi} uses a 2-layer GCN for the semi-supervised classification.

% Certainly the convolutional layers cannot be too many, because repeating Laplacian smoothing too namy times will turn features of different classes closer too as a few edges on the graph connect vertices of different classes. This will make the vertices of different classes indistinguishable.

We apply GCNs with different number of layers on the Zachary's karate club dataset \cite{zachary1977information}, which has 34 vertices of two classes and 78 edges. The GCNs are untrained with the weight parameters initialized randomly as in \cite{glorot2010understanding}. The dimension of the hidden layers is 16, and the dimension of the output layer is 2. The feature vector of each vertex is a one-hot vector. The outputs of each GCN are plotted as two-dimensional points in Fig.~\ref{fig:karate}. We can observe the impact of the graph convolution (Laplacian smoothing) on this small dataset. Applying the smoothing once, the points are not well-separated (Fig.~\ref{fig:sub1}). Applying the smoothing twice, the points from the two classes are separated relatively well. Applying the smoothing again and again, the points are mixed (Fig.~\ref{fig:sub3}, \ref{fig:sub4}, \ref{fig:sub5}). As this is a small dataset and vertices between two classes have quite a number of connections, the mixing happens quickly.

In the following, we will prove that by repeatedly applying Laplacian smoothing many times, the features of vertices within each connected component of the graph will converge to the same values. For the case of symmetric Laplacian smoothing, they will converge to be proportional to the square root of the vertex degree.

Suppose that a graph $\mathcal{G}$ has $k$ connected components $\{C_i\}_{i=1}^{k}$, and the indication vector for the $i$-th component is denoted by $\mathbf{1}^{(i)}\in\mathbb{R}^n$. This vector indicates whether a vertex is in the component $C_i$, i.e.,
\begin{equation}
    \mathbf{1}^{(i)}_j=\left\{
    \begin{array}{l}
        1, v_j \in C_i \\
        0, v_j \not\in C_i
       \end{array} \right.
\end{equation}
\newtheorem{theorem}{Theorem}
\begin{theorem}
    If a graph has no bipartite components, then for any $\mathbf{w}\in \mathbb{R}^n$, and $\alpha \in (0,1]$,
    \begin{align}
        &\lim_{m\to+\infty}{(I-\alpha L_{rw})}^m \mathbf{w} = [\mathbf{1}^{(1)}, \mathbf{1}^{(2)}, \dots, \mathbf{1}^{(k)}]\theta_1,\nonumber \\
        &\lim_{m\to+\infty}{(I-\alpha L_{sym})}^m \mathbf{w}   = D^{-\frac{1}{2}}[\mathbf{1}^{(1)}, \mathbf{1}^{(2)}, \dots, \mathbf{1}^{(k)}]\theta_2,\nonumber
    \end{align}
    where $\theta_1\in \mathbb{R}^k, \theta_2\in \mathbb{R}^k$, i.e., they converge to a linear combination of $\{\mathbf{1}^{(i)}\}_{i=1}^{k}$ and $\{D^{-\frac{1}{2}}\mathbf{1}^{(i)}\}_{i=1}^{k}$ respectively.
\end{theorem}

%i.e. the linear combination of $\{\mathbf{1}^{(i)}\}_{i=1}^{k}$ and $\{D^{-1/2}\mathbf{1}^{(i)}\}_{i=1}^{k}$ respectively

\begin{proof}
    $L_{rw}$ and $L_{sym}$ have the same $n$ eigenvalues (by multiplicity) with different eigenvectors \cite{von2007tutorial}. If a graph has no bipartite components, the eigenvalues all fall in [0,2) \cite{Chung97}. The eigenspaces of $L_{rw}$ and $L_{sym}$ corresponding to eigenvalue 0 are spanned by $\{\mathbf{1}^{(i)}\}_{i=1}^{k}$ and $\{D^{-\frac{1}{2}}\mathbf{1}^{(i)}\}_{i=1}^{k}$
    respectively \cite{von2007tutorial}. For $\alpha\in(0,1]$, the eigenvalues of $(I-\alpha L_{rw})$ and $(I-\alpha L_{sym})$ all fall into (-1,1], and the eigenspaces of eigenvalue 1 are spanned by $\{\mathbf{1}^{(i)}\}_{i=1}^{k}$ and $\{D^{-\frac{1}{2}}\mathbf{1}^{(i)}\}_{i=1}^{k}$ respectively.
    Since the absolute value of all eigenvalues of $(I-\alpha L_{rw})$ and $(I-\alpha L_{sym})$ are less than or equal to 1, after repeatedly multiplying them from the left, the result will converge to the linear combination of eigenvectors of eigenvalue 1, i.e. the linear combination of $\{\mathbf{1}^{(i)}\}_{i=1}^{k}$ and $\{D^{-\frac{1}{2}}\mathbf{1}^{(i)}\}_{i=1}^{k}$ respectively.
\end{proof}
Note that since an extra self-loop is added to each vertex, there is no bipartite component in the graph. Based on the above theorem, over-smoothing will make the features indistinguishable and hurt the classification accuracy.

%Since an extra self-loop is added to each vertex in the graph. So no component can be a bipartite. According to above theorem, over smoothing brings the lose of feature information, which harms the classification accuracy.

The above analysis raises potential concerns about stacking many convolutional layers in a GCN. Besides, a deep GCN is much more difficult to train. In fact, the GCN used in \cite{kipf2016semi} is a 2-layer GCN. However, since the graph convolution is a localized filter -- a linear combination of the feature vectors of adjacent neighbors, a shallow GCN cannot sufficiently propagate the label information to the entire graph with only a few labels. As shown in Fig.~\ref{fig:gcn_validation}, the performance of GCNs (with or without validation) drops quickly as the training size shrinks. In fact, the accuracy of GCNs decreases much faster than the accuracy of label propagation. Since label propagation only uses the graph information while GCNs utilize both structural and vertex features, it reflects the inability of the GCN model in exploring the global graph structure.

Another problem with the GCN model in \cite{kipf2016semi} is that it requires an additional validation set for early stopping in training, which is essentially using the prediction accuracy on the validation set  for model selection. If we optimize a GCN on the training data without using the validation set, it will have a significant drop in performance. As shown in Fig.~\ref{fig:gcn_validation}, the performance of the GCN without validation drops much sharper than the GCN with validation. In \cite{kipf2016semi}, the authors used an additional set of 500 labeled data for validation, which is much more than the total number of training data. This is certainly undesirable as it defeats the purpose of semi-supervised learning. Furthermore, it makes the comparison of GCNs with other methods unfair as other methods such as label propagation may not need the validation data at all.

\section{Solutions}\label{sec:Solution}

We summarize the advantages and disadvantages of the GCN model as follows. The advantages are: 1) the graph convolution -- Laplacian smoothing helps making the classification problem much easier; 2) the multi-layer neural network is a powerful feature extractor. The disadvantages are: 1) the graph convolution is a localized filter, which performs unsatisfactorily with few labeled data; 2) the neural network needs considerable amount of labeled data for validation and model selection.

%We attribute the success of GCNs for semi-supervised learning to two key reasons: 1) The graph convolution -- Laplacian smoothing helps making the classification problem easier; 2) The powerful feature extraction ability of the multi-layer neural networks. However, as shown in Section \ref{sec:Analysis}, the graph convolution only explores local graph structures. When the labeled data is scarce, GCNs will fail because the limited amount of data is not representative enough to train a good classifier.

We want to make best use of the advantages of the GCN model while overcoming its limits. This naturally leads to a co-training \cite{blum1998combining} idea.

\subsection{Co-Train a GCN with a Random Walk Model}

We propose to co-train a GCN with a random walk model as the latter can explore the global graph structure, which complements the GCN model. In particular, we first use a random walk model to find the most confident vertices -- the nearest neighbors to the labeled vertices of each class, and then add them to the label set to train a GCN. Unlike in \cite{kipf2016semi}, we directly optimize the parameters of a GCN on the training set, without requiring additional labeled data for validation.

\begin{algorithm}
    \caption{Expand the Label Set via ParWalks}\label{alg:co-train}
    \label{parwalk}
    \begin{algorithmic}[1]
    \STATE $P:=(L+\alpha \Lambda)^{-1}$
    \FOR{each class $k$}
    \STATE $\bm {p}:=\sum\limits_{j \in \mathcal{S}_k}P_{:,j}$  \\
     \STATE Find the top $t$ vertices in $\bm p$ %$\mathcal{T} := \mathop{\arg\max}\limits_{\mathcal{T}} \sum\limits_{i \in \mathcal{T}}\bm {p}_{i}$, with $|\mathcal{T}|=t$
     \STATE Add them to the training set with label $k$
    % \FOR{each $i$ in $\bar{\mathcal{S}}_k$}
    % \STATE $\hat{d}_{ik}:=\sum_{j\in\mathcal{S}_k}||x_i-x_j||_2$
    % \ENDFOR
    % \STATE Add the $t$ vertices in $\bar{\mathcal{S}}_k$ with smallest $\hat{d}_{ik}$ to class $k$ to the training set.
    \ENDFOR
    \end{algorithmic}
\end{algorithm}

We choose to use the partially absorbing random walks (ParWalks) \cite{Wu12parw} as our random walk model. A partially absorbing random walk is a second-order Markov chain with partial absorption at each state. It was shown in \cite{nips13_harmonic} that with proper absorption settings, the absorption probabilities can well capture the global graph structure. Importantly, the absorption probabilities can be computed in a closed-form by solving a simple linear system, and can be fast approximated by random walk sampling or scaled up on top of vertex-centric graph engines \cite{guo2017graph}.

%Introduced by Wu et al. in \cite{Wu12parw}, partially absorbing random walks (ParWalk) is a second-order Markov chain which generalizes the absorbing random walks. In ParWalk, whenever a token reaches a new state $i$, it will be absorbed at $i$ with some probability $p_{ii}$, and follows a random edge out of $i$ with probability $1-p_{ii}$. Once the token is absorbed, it will never get out. Eventually, the token  will get absorbed at some state. The authors of \cite{Wu12parw} proposed to use the probability that a token starting from state $i$ gets absorbed at state $j$ to measure the proximity of vertices $i$ and $j$ on the underlying graph. In \cite{Wu12parw} and \cite{wu2015new}, the authors showed that by setting $\alpha$ and $\Lambda$, $A$ unifies several popular proximity measures such as Personalized PageRank and hitting times, hence it is interesting to study the properties of $A$ for model comparison and selection. In \cite{Wu12parw}, the authors showed that $A$ can well encode the graph structure when $\Lambda=I$ and $\alpha$ is small. In \cite{nips13_harmonic}, the authors generalized the results in \cite{Wu12parw} and showed that for almost any $\Lambda$, using columns of $A$ would be equally desirable when $\alpha$ is small. However, the analysis in both \cite{Wu12parw} and \cite{nips13_harmonic} are not clear enough to explain why $\alpha$ should be set small. In the following, we will give a new analysis to clarify this issue.

% (Now describe the algorithm is described in Algorithm \ref{parwalk}).

The algorithm to expand the training set is described in Algorithm \ref{alg:co-train}.
First, we calculate the normalized absorption probability matrix $P=(L+\alpha \Lambda)^{-1}$ (the choice of $\Lambda$ may depend on data). $P_{i,j}$ is the probability of a random walk from vertex $i$ being absorbed by vertex $j$, which represents how likely $i$ and $j$ belong to the same class. Second, we need to measure the confidence of a vertex belonging to class $k$. We partition the labeled vertices into ${\mathcal{S}_1, \mathcal{S}_2, ...}$, where $\mathcal{S}_k$ denotes the set of labeled data of class $k$. For each class $k$, we calculate a confidence vector $\mathbf{p} = \sum\limits_{j \in \mathcal{S}_k}P_{:,j}$,
where $\mathbf{p} \in \mathbb{R}^n$ and $p_{i}$ is the confidence of vertex $i$ belonging to class $k$. Finally, we find the $t$ most confident vertices and add them to the training set with label $k$ to train a GCN.

%Note that the computation of the absorption probabilities can be easily scaled up on top of vertex-centric graph engines \cite{guo2017graph}.

\subsection{GCN Self-Training}

Another way to make a GCN ``see'' more training examples is to self-train a GCN. Specifically, we first train a GCN with given labels, then select the most confident predictions for each class by comparing the softmax scores, and add them to the label set. We then continue to train the GCN with the expanded label set, using the pre-trained GCN as initialization. This is described in Algorithm \ref{alg:self-train}.

% (by comparing their softmax score? explain )

The most confident instances found by the GCN are supposed to share similar (but not the same) features with the labeled data. Adding them to the labeled set will help training a more robust and accurate classifier. Furthermore, it complements the co-training method in the situation that a graph has many isolated small components and it is not possible to propagate labels with random walks.

\begin{algorithm}
    \caption{Expand the Label Set via Self-Training}\label{alg:self-train}
    \label{parwalk}
    \begin{algorithmic}[1]
    \STATE $\bm Z:=GCN(X) \in \mathbb{R}^{n\times F}$, the output of GCN
    \FOR{each class $k$}
    \STATE Find the top $t$ vertices in $Z_{i,k}$
    %$\mathcal{S} := \mathop{\arg\max}_{\mathcal{S}} \sum_{i \in \mathcal{T}}\bm Z_{i,k}$, with $|\mathcal{S}|=t$
    \STATE Add them to the training set with label $k$
    % \FOR{each $i$ in $\bar{\mathcal{S}}_k$}
    % \STATE $\hat{d}_{ik}:=\sum_{j\in\mathcal{S}_k}||x_i-x_j||_2$
    % \ENDFOR
    % \STATE Add the $t$ vertices in $\bar{\mathcal{S}}_k$ with smallest $\hat{d}_{ik}$ to class $k$ to the training set.
    \ENDFOR
    \end{algorithmic}
\end{algorithm}

%GCNs fail because the limited amount of training data is not representative enough to train a good classifier in most cases. Our idea is to make GCNs ``see'' more training examples by using co-training and self-learning.

\textbf{Combine Co-Training and Self-Training.} To improve the diversity of labels and train a more robust classifier, we propose to combine co-training and self-learning. Specifically, we expand the label set with the most confident predictions found by the random walk and those found by the GCN itself, and then use the expanded label set to continue to train the GCN. We call this method ``Union''. To find more accurate labels to add to the labeled set, we also propose to add the most confident predictions found by both the random walk and the GCN. We call this method ``Intersection''.

%The algorithm is described as follows: 1) training a GCN; 2) expanding the labeled set by random walks; 3) expand the labeled set by GCN; 4) continuing to train the GCN with the expanded labeled set. We do not need additional validation data as stated in \cite{kipf2016semi}, which are used to determine when to stop training.

Note that we optimize all our methods on the expanded label set, without requiring any additional validation data. As long as the expanded label set contains enough correct labels, our methods are expected to train a good GCN classifier. But how much labeled data does it require to train a GCN? Suppose that the number of layers of the GCN is $\tau$, and the average degree of the underlying graph is $\hat{d}$. We propose to estimate the lower bound of the number of labels $\eta=|\mathcal{V}_l|$ by solving  $(\hat{d})^{\tau}*\eta \approx n.$ The rationale behind this is to estimate how many labels are needed to for a GCN with $\tau$ layers to propagate them to cover the entire graph.

%the labeling percentage is $\eta = \frac{|\mathcal{V}_l|}{n}$,
% Elaborate more on this point.

%We then use the additional ``predicted'' labels with given true labels to co-train a GCN. The algorithm is described as follows.

%There is a sensitive parameter in Algorithm \ref{alg:co-train}, namely the number of labels added to each class. A more fundamental question is how many labeled data does a GCN need? Suppose that the labeling percentage is $\eta$, the number of layers of GCN is $\gamma$, and the average degree of the underlying graph is $\hat{d}$. We propose to estimate $\eta$ by $(\hat{d})^{\gamma}*\eta \approx 1.$ Then $\eta$ extra labels are distributed to each class, and the vertices added to class $k$ are proportional to the size of class $k$ in the training set.

 %The number of neighbors is not fixed, but determined by the labeling percentage (training size in GCN) and the number of layers of GCN. Denote by the labeling percentage as $\eta$, the number of layers of GCN as $\gamma$, and the size of the neighborhood as $\kappa$. Assume $$(\kappa)^{\gamma}*\eta \approx 1.$$ Then we have
%$$\kappa=\ceil{\sqrt[\gamma]{\frac{1}{\eta}}}.$$

\section{Experiments}\label{sec:Experiment}

In this section, we conduct extensive experiments on real benchmarks to verify our theory and the proposed methods, including Co-Training, Self-Training, Union, and Intersection (see Section \ref{sec:Solution}).

%\begin{itemize}
%\item \textbf{Co-training}, which expands training set by Random Walk.
%\item \textbf{GCN self-training}. which expands training set by GCN itself.
%\item \textbf{Intersection} of Self-training and Co-training, which expands training set by the intersection of the vertices added by GCN and the ones added by Random Walk.
%\item \textbf{Union} of Self-training and Co-training, which expands training set by the union of the vertices added by GCN and the ones added by Random Walk.
%\end{itemize}��

% GCN LP
%In this section, we conduct extensive experiments on real benchmarks to verify our theory and methods. We test the following methods proposed by us. Co-training (expanding training set by Random Walk), GCN self-training (expanding training set by GCN itself), Intersection of Self-training and Co-training (expand training set by the intersection of the vertices added by GCN and the ones added by Random Walk), Union of Self-training and Co-training (expand training set by the union of the vertices added by GCN and the ones added by Random Walk).

% (explain clearly the following methods and give them specific names) including Co-Training GCN with Random Walk (give it a name, A?), GCN self-training (model B?), intersection of model A and B , union of model A and B (algorithm 1)...

We compare our methods with several state-of-the-art methods, including
GCN with validation (GCN+V); GCN without validation (GCN-V); GCN with Chebyshev filter (Cheby) \cite{kipf2016semi}; label propagation using ParWalks (LP) \cite{Wu12parw}; Planetoid \cite{yang2016revisiting}; DeepWalk \cite{perozzi2014deepwalk}; manifold regularization (ManiReg) \cite{belkin2006manifold}; semi-supervised embedding (SemiEmb) \cite{weston2012deep}; iterative classification algorithm (ICA) \cite{sen2008collective}.

\subsection{Experimental Setup}

We conduct experiments on three commonly used citation networks: CiteSeer, Cora, and PubMed \cite{sen2008collective}. The statistics of the datasets are summarized in \tablename\;\ref{tab:dataset}. On each dataset, a document is described by a bag-of-words feature vector, i.e., a 0/1-valued vector indicating the absence/presence of a certain word. The citation links between documents are described by a 0/1-valued adjacency matrix. The datasets we use for testing are provided by the authors of \cite{yang2016revisiting} and \cite{kipf2016semi}.

%PubMed stores million of citations for biomedical literature and the links to their full-text content. Cora dataset consists of 2708 scientific publications classified into one of seven classes and 5429 links between the publications. CiteSeer is a collection of computer bibliographies. The documents of CiteSeer are classified into six classes, which are miscellaneous, techreport, dissertation, article, inproceedings and mastersthesis. The Pubmed and CiteSeer dateset we used is actually a subset of original Pubmed and CiteSeer dataset, provided by \cite{kipf2016semi}.

For ParWalks, we set $\Lambda = I$, and $\alpha=10^{-6}$, following \citeauthor{Wu12parw}. For GCNs, we use the same hyper-parameters as in \cite{kipf2016semi}: a learning rate of $0.01$, $200$ maximum epochs, $0.5$ dropout rate, $5\times10^{-4}$ $L2$ regularization weight, 2 convolutional layers, and $16$ hidden units, which are validated on Cora by \citeauthor{kipf2016semi}. For each run, we randomly split labels into a small set for training, and a set with 1000 samples for testing. For GCN+V, we follow \cite{kipf2016semi} to sample additional 500 labels for validation. For GCN-V, we simply optimize the GCN using training accuracy. For Cheby, we set the polynomial degree $K=2$ (see Eq.~(\ref{eq:cheby})). We test these methods with 0.5\%, 1\%, 2\%, 3\%, 4\%, 5\% training size on Cora and CiteSeer, and with 0.03\%, 0.05\%, 0.1\%, 0.3\% training size on PubMed. We choose these labeling rates for easy comparison with \cite{kipf2016semi}, \cite{yang2016revisiting}, and other methods. We report the mean accuracy of 50 runs except for the results on PubMed \cite{yang2016revisiting}, which are averaged over 10 runs.

%Cora and CiteSeer contains comparable numbers of vertices, but PubMed contains far more vertices than the other two. These designed label rates keep absolute label numbers of different datasets in the same magnitude.

\begin{table}
\caption{Dataset statistics}\label{tab:dataset}
\centering
\begin{tabular}{lrrrr}
\textbf{Dataset}    &    \textbf{Nodes}    &    \textbf{Edges}    &    \textbf{Classes}    &    \textbf{Features} \\
\midrule
CiteSeer    &    3327    &    4732    &    6        &    3703    \\
Cora     &    2708    &    5429    &    7        &    1433    \\
PubMed    &    19717    &    44338    &    3        &    500    \\
\end{tabular}
\end{table}

%From each dataset, We random pick out a small train set and then a validation set with 1000 labeled samples from the rest.

%We follow the experimental setup in Kipf \cite{kipf2016semi}.
%
% From each dataset, We random pick out a small train set and then a validation set with 1000 labeled samples from the rest. Kipf \cite{kipf2016semi} also used a validation set of size 500, so we also condsider the condition that no validation set is available. We just choose the model which performs best on train set in order to abandon validation set.
%
%We use the same hyper-parameters as \cite{kipf2016semi}: a learning rate of 0.01, 200 maximum epoch, 0.5 dropout rate, $5\times10^{-4}$ L2 regularization weight and 16 hidden units. For ParWalk, we set $\alpha=10^{-6}$, $s=100$, $t=50$.

\subsection{Results Analysis}

The classification results are summarized in \tablename\;\ref{tab:cora}, \ref{tab:citeseer} and \ref{tab:pubmed}, where the highest accuracy in each column is highlighted in bold and the top 3 are underlined. Our methods are displayed at the bottom half of each table.

We can see that the performance of Co-Training is closely related to the performance of LP. If the data has strong manifold structure, such as PubMed, Co-Training performs the best. In contrast, Self-Training is the worst on PubMed, as it does not utilize the graph structure. But Self-Training does well on CiteSeer where Co-Training is overall the worst. Intersection performs better when the training size is relatively large, because it filters out many labels. Union performs best in many cases since it adds more diverse labels to the training set.

\textbf{Comparison with GCNs.}
% Explain how our methods outperform GCNs by a large margin while without requiring any validation data.
At a glance, we can see that on each dataset, our methods outperform others by a large margin in most cases. When the training size is small, all our methods are far better than GCN-V, and much better than GCN+V in most cases. For example, with labeling rate 1\% on Cora and CiteSeer, our methods improve over GCN-V by 23\% and 28\%, and improve over GCN+V by 12\% and 7\%. With labeling rate 0.05\% on PubMed, our methods improve over GCN-V and GCN+V by 37\% and 18\% respectively. This verifies our analysis that the GCN model cannot effectively propagate labels to the entire graph with small training size. When the training size grows, our methods are still better than GCN+V in most cases, demonstrating the effectiveness of our approaches. When the training size is large enough, our methods and GCNs perform similarly, indicating that the given labels are sufficient for training a good GCN classifier. Cheby does not perform well in most cases, which is probably due to overfitting.

\begin{table}[t]
\centering
\small
\caption{Classification Accuracy On Cora}\label{tab:cora}
\begin{tabular}{l cccccc}
    \multicolumn{7}{c}{ \normalsize \textbf{Cora}} \\
    % \cmidrule{2-5}\cmidrule{7-10}\cmidrule{12-15}
    \textbf{Label Rate} & 0.5\% & 1\% & 2\% & 3\% & 4\% & 5\% \\
    \midrule
    \textbf{LP}             & \underline{56.4} & 62.3 & 65.4 & 67.5 & 69.0 & 70.2 \\
    % LP-self & 52.2         & 59.5 & 65.0 & 66.8 & 68.2 & 70.0 \\
    \textbf{Cheby}          & 38.0 & 52.0 & 62.4 & 70.8 & 74.1 & 77.6 \\
    \textbf{GCN-V}          & 42.6 & 56.9 & 67.8 & 74.9 & 77.6 & 79.3 \\
    \textbf{GCN+V}          & 50.9 & 62.3 & 72.2 & 76.5 & 78.4 & 79.7 \\
    \midrule
    \textbf{Co-training}    & \underline{56.6} & \underline{66.4} & \underline{73.5} & 75.9 & 78.9 & \underline{80.8} \\
    \textbf{Self-training}  & 53.7 & \underline{66.1} & \underline{73.8} & \underline{77.2} & \underline{79.4} & 80.0 \\
    \textbf{Union}          & \underline{\textbf{58.5}} & \underline{\textbf{69.9}} & \underline{\textbf{75.9}} & \underline{\textbf{78.5}} & \underline{\textbf{80.4}} & \underline{\textbf{81.7}} \\
    \textbf{Intersection}   & 49.7 & 65.0 & 72.9 & \underline{77.1} & \underline{79.4} & \underline{80.2} \\
    % \textbf{ParWalk self-training} \\
    % \textbf{Planetoid-T} \\
\end{tabular}
\end{table}

\begin{table}[t]
\small
\centering
\caption{Classification Accuracy on CiteSeer}\label{tab:citeseer}
\begin{tabular}{l cccccc}
    \multicolumn{7}{c}{\normalsize \textbf{CiteSeer}}\\
    % \cmidrule{2-5}\cmidrule{7-10}\cmidrule{12-15}
    \textbf{Label Rate} & 0.5\% & 1\% & 2\% & 3\% & 4\% & 5\% \\
    \midrule
    \textbf{LP}             & 34.8 & 40.2 & 43.6 & 45.3 & 46.4 & 47.3 \\
    % LP-self & 52.2       & 38.1 & 45.2 & 47.4 & 47.8 & 48.1 & 49.6 \\
    \textbf{Cheby}          & 31.7 & 42.8 & 59.9 & 66.2 & 68.3 & 69.3 \\
    \textbf{GCN-V}          & 33.4 & 46.5 & 62.6 & 66.9 & 68.4 & 69.5 \\
    \textbf{GCN+V}          & \underline{43.6} & 55.3 & 64.9 & \underline{67.5} & \underline{68.7} & \underline{69.6} \\
    \midrule
    \textbf{Co-training}    & \underline{\textbf{47.3}} & 55.7 & 62.1 & 62.5 & 64.5 & 65.5 \\
    \textbf{Self-training}  & 43.3 & \underline{58.1} & \underline{68.2} & \underline{69.8} & \underline{70.4} & \underline{71.0} \\
    \textbf{Union}          & \underline{46.3} & \underline{\textbf{59.1}} & \underline{66.7} & 66.7 & 67.6 & 68.2 \\
    \textbf{Intersection}   & 42.9 & \underline{\textbf{59.1}} & \underline{\textbf{68.6}} & \underline{\textbf{70.1}} & \underline{\textbf{70.8}} & \underline{\textbf{71.2}} \\
    % \textbf{ParWalk self-training} \\
    % \textbf{Planetoid-T} \\
\end{tabular}
\end{table}

\begin{table}[t]
\centering
\caption{Classification Accuracy On PubMed}\label{tab:pubmed}
\normalsize
\begin{tabular}{l cccc}
    \multicolumn{5}{c}{\normalsize \textbf{PubMed}} \\
    % \cmidrule{2-5}\cmidrule{7-10}\cmidrule{12-15}
    \textbf{Label Rate} & 0.03\% & 0.05\% & 0.1\% & 0.3\% \\
    \midrule
    \textbf{LP}             & \underline{61.4} & \underline{66.4} & 65.4 & 66.8 \\
    % LP-self & 52.2       & 58.2 & 64.7 & 65.8 & 67.9 \\
    \textbf{Cheby}          & 40.4 & 47.3 & 51.2 & 72.8 \\
    \textbf{GCN-V}          & 46.4 & 49.7 & 56.3 & 76.6 \\
    \textbf{GCN+V}          & \underline{60.5} & 57.5 & 65.9 & \underline{77.8} \\
    \midrule
    \textbf{Co-training}    & \underline{\textbf{62.2}} & \underline{\textbf{68.3}} & \underline{\textbf{72.7}} & \underline{78.2} \\
    \textbf{Self-training}  & 51.9 & 58.7 & 66.8 & 77.0 \\
    \textbf{Union}          & 58.4 & \underline{64.0} & \underline{70.7} & \underline{\textbf{79.2}} \\
    \textbf{Intersection}   & 52.0 & 59.3 & \underline{69.4} & 77.6 \\
    % \textbf{ParWalk self-training} \\
    % \textbf{Planetoid-T} \\
\end{tabular}
\end{table}

\begin{table}[t]
    \caption{Accuracy under 20 Labels per Class}\label{tab:gcn_baseline}
    \centering
    \begin{tabular}{l ccc}
    \textbf{Method} & \textbf{CiteSeer} & \textbf{Cora} & \textbf{Pubmed} \\
    \midrule
    \textbf{ManiReg}   & 60.1 & 59.5 & 70.7 \\
    \textbf{SemiEmb}   & 59.6 & 59.0 & 71.7 \\
    \textbf{LP}        & 45.3 & 68.0 & 63.0 \\
    \textbf{DeepWalk}  & 43.2 & 67.2 & 65.3 \\
    \textbf{ICA}       & \underline{69.1} & 75.1 & 73.9 \\
    \textbf{Planetoid} & 64.7 & 75.7 & \underline{77.2} \\
    \textbf{GCN-V}     & 68.1 & 80.0 & 78.2 \\
    \textbf{GCN+V}     & \underline{68.9} & \underline{80.3} & \underline{\textbf{79.1}} \\
    \midrule
    \textbf{Co-training}    & 64.0 & 79.6 & 77.1 \\
    \textbf{Self-training}  & 67.8 & \underline{80.2} & 76.9 \\
    \textbf{Union}          & 65.7 & \underline{\textbf{80.5}} & \underline{78.3} \\
    \textbf{Intersection}   & \underline{\textbf{69.9}} & 79.8 & 77.0 \\

    \end{tabular}
\end{table}

%When the train size is very small, the Intersection method tends to perform worse as
%
%expand training set more accurately, it provided only a few new labels. The ability of LP to exploring global graph topology does help GCNs a lot to see more graph structure. The intersection of co-training and self-training expand training set more accurately, i.e. less wrong labels are added. The Union of co-training and self-training attaches more diversity to train set. So they can achieve heigher accuracy than co-training alone.

%We can see at a glance that our models ourperform others by a large margin. As discussion in section \ref{sec:Analysis}, GCNs only explore the local graph structure.

% The second best model is
% They all outperform others significantly. ParWalk+GCN with validation is the best on Cora and PubMed by a large margin, and only slightly worse than GCN with validation on citeseer. This can be explained by the performance of ParWalk. As ParWalk has much higher accuracy on Cora and PubMed than on CiteSeer, the additional labels found by ParWalk on CiteSeer may be less accurate thus degrade the performance of GCN. As label propagation with ParWalk does not use the vertex feature, it is not surprising that its performance is lower than others.

\textbf{Comparison with other methods.} We compare our methods with other state-of-the-art methods in \tablename\;\ref{tab:gcn_baseline}. The experimental setup is the same except that for every dataset, we sample 20 labels for each class, which corresponds to the total labeling rate of 3.6\% on CiteSeer, 5.1\% on Cora, and 0.3\% on PubMed. The results of other baselines are copied from \cite{kipf2016semi}. Our methods perform similarly as GCNs and outperform other baselines significantly. Although we did not directly compare with other baselines, we can see from \tablename\;\ref{tab:cora}, \ref{tab:citeseer} and \ref{tab:pubmed} that our methods with much fewer labels already outperform many baselines. For example, our method Union on Cora (\tablename\;\ref{tab:cora}) with 2\% labeling rate (54 labels) beats all other baselines with 140 labels (\tablename\;\ref{tab:gcn_baseline}).

\textbf{Influence of the Parameters.} A common parameter of our methods is the number of newly added labels. Adding too many labels will introduce noise, but with too few labels we cannot train a good GCN classifier. As described in the end of Section \ref{sec:Solution}, we can estimate the lower bound of the total number of labels $\eta$ needed to train a GCN by solving $(\hat{d})^\tau*\eta\thickapprox n$. We use $3\eta$ in our experiments. Actually, we found that $2\eta$, $3\eta$ and $4\eta$ perform similarly in the experiments. We follow \citeauthor{kipf2016semi} to set the number of convolutional layers as 2. We also observed in the experiments that 2-layer GCNs performed the best. When the number of convolutional layers grows, the classification accuracy decreases drastically, which is probably due to overfitting.

\textbf{Computational Cost.} For Co-Training, the overhead is the computational cost of the random walk model, which requires solving a sparse linear system. In our experiments, the time is negligible on Cora and CiteSeer as there are only a few thousand vertices. On PubMed, it takes less than 0.38 seconds in MatLab R2015b. As mentioned in Section \ref{sec:Solution}, the computation can be further speeded up using vertex-centric graph engines (Guo et al. 2017), so the scalability of our method is not an issue. For Self-Training, we only need to run a few epochs in addition to training a GCN. It converges fast as it builds on a pre-trained GCN. Hence, the running time of Self-Training is comparable to a GCN.

\section{Conclusions}\label{sec:conclusion}

Understanding deep neural networks is crucial for realizing their full potentials in real applications. This paper contributes to the understanding of the GCN model and its application in semi-supervised classification. Our analysis not only reveals the mechanisms and limitations of the GCN model, but also leads to new solutions overcoming its limits. In future work, we plan to develop new convolutional filters which are compatible with deep architectures, and exploit advanced deep learning techniques to improve the performance of GCNs for more graph-based applications.

\section*{Acknowledgments}

This research received support from the grant 1-ZVJJ funded by the Hong Kong Polytechnic University. The authors would like to thank the reviewers for their insightful comments and useful discussions.

\bibliographystyle{aaai}
\bibliography{parw}

\end{document}